\documentclass[runningheads]{llncs}

\usepackage[normalem]{ulem}
\usepackage[greek,english]{babel}
\usepackage{amsmath}
\usepackage{amsfonts}

\usepackage{flushend}
\usepackage{xcolor}
\usepackage[pdftex]{graphicx}
\usepackage{float}
\usepackage{caption, subcaption}
\usepackage[export]{adjustbox}
\usepackage{wrapfig}

\usepackage{amssymb}% http://ctan.org/pkg/amssymb
\usepackage{pifont}% http://ctan.org/pkg/pifont
\usepackage{multirow}
\usepackage{multicol}
\usepackage{arydshln}
\usepackage[edges]{forest}
\usepackage{nccmath}
\usepackage{booktabs}
\usepackage{caption}
\usepackage{tabularx}
\usepackage{booktabs}
\usepackage{caption}
\usepackage{geometry}
\usepackage{tabularray}
\usepackage{bbm}
\usepackage{dsfont}

\newcommand{\eat}[1]{}

\usepackage[official]{eurosym}
\usepackage{enumitem}
\setlist{nolistsep,noitemsep}
\usepackage{hyperref}

\begin{document}

\title{From Log-Odds to Shapley Values: \\An Explanatory Geometry for the Weighted Naive Bayes Classifier}

\titlerunning{From Log-Odds to Shapley Values}

\author{Vincent Lemaire\inst{1} \and
Fabrice Clérot\inst{2} 
}

\authorrunning{}

\institute{
Orange Research, France, \email{vincent.lemaire@orange.com}\and
Orange Research, France, \email{fabrice.clerot@orange.com}
}

\maketitle

\begin{abstract}
This paper studies the construction of an explanatory space for a weighted naive Bayes classifier from the supervised representation induced by the model. We start from the classical supervised distance based on conditional log-likelihoods and introduce a discriminative reformulation based on log-odds, which is more directly related to the classification decision. We then show that this representation induces a distance that exactly coincides with the $\ell_1$ distance between vectors of analytical Shapley values, thereby providing a formal explanatory interpretation of the geometry induced by the model. Finally, we empirically compare several supervised distances derived from these representations using a $k$-nearest neighbors classifier. This work highlights a close link between supervised distance, local explanation, and predictive behavior, from a primarily methodological perspective.
\end{abstract}

\section{Context: Supervised distance induced by a classifier}
\label{distance}

\subsection{Introduction}

%Dans le cadre d’un classifieur probabiliste, il peut être utile de disposer d’une représentation des observations cohérente avec la manière dont le modèle produit ses scores de prédiction. Une telle représentation permet de définir une distance entre individus qui ne repose pas uniquement sur la géométrie brute des variables d’entrée, mais sur leur rôle dans la décision du classifieur. Dans le cas du classifieur naïf de Bayes cette idée conduit naturellement à construire une représentation supervisée fondée sur les probabilités conditionnelles associées aux variables explicatives. Deux observations proches dans cette représentation ont alors vocation à être proches également du point de vue de leur comportement prédictif. Cette section rappelle le principe de cette représentation induite par le classifieur, ainsi que la distance qui lui est associée. Cette construction servira ensuite de point de départ pour introduire une représentation plus directement discriminante fondée sur les log-odds en Section \ref{logodds}, puis pour établir un lien avec les valeurs de Shapley analytiques du naïf de Bayes en Sections \ref{sec-shap} et \ref{sec-multi}. Enfin, avant de conclure, une validation expérimentale avec un classifieur  fondé sur l'algorithme des $k$ plus proches voisins ($k$-NN) est présentée en Section~\ref{expe}.

In the context of a probabilistic classifier, it can be useful to have a
representation of the observations that is consistent with how the model generates its prediction scores. Such a representation allows us to define a distance between
individuals that is based not only on the raw geometry of the input variables,
but also on their role in the classifier’s decision. In the case of the naive Bayes classifier, this idea naturally leads to constructing a supervised representation based on the conditional probabilities associated with the explanatory variables. Two observations that are close in this representation are then expected to be close in terms of their predictive behavior as well. This section reviews the principle of this classifier-induced representation, as well as the distance associated with it. This construction will then serve as a starting point for introducing a more directly discriminative representation based on log-odds in Section \ref{logodds}, and subsequently for establishing a connection with the analytical Shapley values of the naive Bayes classifier in Sections \ref{sec-shap} and \ref{sec-multi}. Finally, before concluding, an experimental validation using a classifier  based on the $k$-nearest-neighbor ($k$-NN) algorithm is presented in Section~\ref{expe}.

%\subsection{Prétraitement}
\subsection{Preprocessing}
\label{Preprocessing}
% Le classifieur naïf de Bayes utilisé dans la librairie Khiops \cite{khiops} repose sur une étape préalable de préparation des variables explicatives. Les variables numériques sont discrétisées, tandis que les variables catégorielles font l’objet d’un groupement de modalités. À l’issue de ce prétraitement, chaque variable $m$ est recodée en une variable discrète comportant $I_m$ valeurs possibles. Une observation \(D\) peut ainsi être représentée sous la forme
% $$
% \small
% D = (D_{1 i_1}, D_{2 i_2}, \dots, D_{M i_M}),
% $$
% où \(D_{m i_m}\) désigne la valeur recodée prise par la variable \(m\) pour cette
% observation. Dans la suite de l'article par souci de simplification de lecture $\small
% D = (D_{1 i_1}, D_{2 i_2}, \dots, D_{M i_M}),
% $ sera noté $
% D = (D_{1}, D_{2}, \dots, D_{M})
% $.

The naive Bayes classifier used in the Khiops library \cite{khiops} relies on a preliminary step to prepare the explanatory variables. Numeric variables are discretized, while categorical variables are grouped by category. Following this preprocessing, each variable $m$ is recoded into a discrete variable with $I_m$ possible values. An observation \(D\) can thus be represented as
$$
D = (D_{1 i_1}, D_{2 i_2}, \dots, D_{M i_M}),
$$
where \(D_{m i_m}\) denotes the recoded value taken by the variable \(m\) for this
observation. For the sake of simplicity in the remainder of this article, $D = (D_{1 i_1}, D_{2 i_2}, \dots, D_{M i_M}),
$ will be denoted by $
D = (D_{1}, D_{2}, \dots, D_{M})
$.

%Ce prétraitement permet d’estimer, pour chaque variable recodée et pour chaque classe cible $(C_j)$, les probabilités conditionnelles $p(D_{m}|C_j)$, qui constituent les briques élémentaires du score du classifieur naïf de Bayes. Il fournit ainsi une représentation homogène des variables, adaptée à la définition d’une distance supervisée cohérente avec le modèle.

This preprocessing enables us to estimate the conditional probabilities $p(D_m|C_j)$ for each recoded variable and target class $C_j$. These probabilities form the basis of the naive Bayes classifier’s score. This provides a homogeneous representation of the variables, suitable for defining a supervised distance consistent with the model.

%Autrement dit, la description d’une observation ne se limite plus à ses valeurs brutes : elle peut être exprimée à travers les quantités probabilistes utilisées par le classifieur pour produire sa décision. C’est cette représentation induite par le modèle qui sera utilisée dans la suite pour définir une distance entre observations.

In other words, an observation can be described not only by its raw values, but also by the probabilistic quantities used by the classifier to make a decision. This model-induced representation will be used to define a distance between observations in what follows.

%\subsection{Distance dépendante des classes cibles}
\subsection{Distance depending on target classes}
\label{calculdistance}

% Lemaire et al. \cite{Lemaire2015} ont proposé de définir une distance supervisée induite par un classifieur ici le classifieur naïf de Bayes; distance cohérente avec la manière dont le classifieur calcule ses scores de classification. En reprenant les notations introduites Section 1.2 on ré-écrit, pour chaque classe cible $C_j$, le score logarithmique du classifieur naïf de Bayes sous la forme
% \begin{equation}
% \small
% \log p(C_j \mid D)=\sum_{m=1}^M \log p(D_{m}\mid C_j)+\log p(C_j)-\log p(D),
% \label{eq:nb_logposterior}
% \end{equation}
% où \(D=(D_m)_{m=1,\dots,M}\) désigne une instance et \(D_{m}\) la valeur recodée
% de la variable \(m\).

Lemaire et al. \cite{Lemaire2015} proposed defining a supervised distance induced by a classifier—in this case, the Naive Bayes classifier—that is consistent with how the classifier computes its classification scores. Using the notation introduced in Section~\ref{Preprocessing}, we can rewrite, for each target class $C_j$, the logarithmic score of the naive Bayes classifier in the form
\begin{equation}
\log p(C_j \mid D)=\sum_{m=1}^M \log p(D_{m}\mid C_j)+\log p(C_j)-\log p(D),
\label{eq:nb_logposterior}
\end{equation}
where $D=(D_m)_{m=1,\dots,M}$ denotes an instance and $D_{m}$ the recoded value
of variable $m$.

% La décision bayésienne consiste à choisir la classe $C_j$ maximisant cette quantité.
% Lemaire et al. \cite{Lemaire2015} définissent alors la distance entre
% deux instances $D$ et $D'$ par

The Bayesian decision is to choose the class $C_j$ that maximizes this quantity.
Lemaire et al. \cite{Lemaire2015} then define the distance between
two instances $D$ and $D'$ as
\begin{equation}
d_{NB}^{1}(D,D')
=
\sum_{m=1}^M \sum_{j=1}^J
\left|
\log p(D_{m}\mid C_j)
-
\log p(D'_{m}\mid C_j)
\right|.
\label{eq:d1nb_original}
\end{equation}

%Chaque instance est ainsi représentée par un vecteur dans un espace supervisé de dimension $MJ$, dont les composantes sont les quantités $\log p(D_m \mid C_j)$, indexées par les couples $(m,j)$ avec $m \in \{1,\ldots,M\}$ et $j \in \{1,\ldots,J\}$. La distance proposée correspond à la norme $\ell_1$ entre ces vecteurs. La distance entre deux instances est donc définie en fonction des recodages. Deux instances proches au sens de leur recodage supervisé seront proches au sens de leur comportement pour la classe à prédire. En effet si on définit la distance entre les distributions de classes prédites de la façon suivante :

Each instance is thus represented by a vector in a supervised space of dimension $MJ$, whose components are the quantities $\log p(D_m \mid C_j)$, indexed by the pairs $(m,j)$ where $m \in \{1,\ldots,M\}$ and $j \in \{1,\ldots,J\}$. The proposed distance corresponds to the $\ell_1$ norm between these vectors. The distance between two instances is therefore defined in terms of their supervised recodings. Instances that are close in terms of their supervised recodings will exhibit similar behaviour with respect to the predicted class. Indeed, if we define the distance between the predicted class distributions as follows:
\begin{align}
\Delta^1(D,D')=\sum_{j=1}^{J} \left| log(p(C_j|D))-log(p(C_j|D')) \right| 
\end{align}

Thus we have the following majorization \cite{Lemaire2015}
%On a la majoration suivante:
\begin{align}\label{majoration_nb}
\Delta^1(D,D') \leq \left[ d_{NB}^{1}(D,D') + J \left| log(p(D))-log(p(D')) \right| \right]
%\mbox{avec }  p(D)= \prod_m p(D_{mi_m}) = \prod_m  \sum_j p(D_{mi_m}|C_j) p(C_j) \nonumber
\end{align}

%Cette inégalité montre que la distance $d_1^{\mathrm{NB}}$ contrôle les écarts entre log-probabilités a posteriori à un terme de normalisation près. Ainsi, lorsque deux observations sont proches au sens de $d_1^{\mathrm{NB}}$ et que leurs termes de normalisation $\log p(D)$ et $\log p(D')$ sont également proches, leurs vecteurs de scores prédictifs sont proches au sens de $\Delta_1$. Dans le cadre du modèle naïf de Bayes génératif, le terme $\log p(D)$ est lui-même déterminé par les scores de classes via une opération de type \emph{log-sum-exp}. Il est donc lié aux log-vraisemblances conditionnelles utilisées dans la représentation, même si l'équation~(\ref{majoration_nb}) isole explicitement sa contribution.

This inequality shows that the distance $d_1^{\mathrm{NB}}$ controls the differences between posterior log-probabilities up to a normalization term. Thus, when two observations are close in the sense of $d_1^{\mathrm{NB}}$ and their normalization terms $\log p(D)$ et $\log p(D')$ are also close, their predictive score vectors are close in the sense of $\Delta_1$. In the context of the generative naive Bayes model, the term $\log p(D)$ is itself determined by the class scores via a \emph{log-sum-exp} operation. Therefore, despite equation (\ref{majoration_nb}) explicitly isolating its contribution, it is related to the conditional log-likelihoods used in the representation.

%%%%%%%%%%%%%%%%%%%%%%

%Dans le cas d'un classifieur naïf de Bayes `pondéré' les variables reçoivent des poids tel que l'équation \ref{eq:d1nb_original} devient :

In the case of a `weighted' Naive Bayes classifier, the variables are assigned weights such that the original equation \ref{eq:d1nb_original} becomes:
\begin{equation}
d_{SNB}^{1}(D,D')
=
\sum_{m=1}^M \sum_{j=1}^J
w_m
\left|
\log p(D_{m}\mid C_j)
-
\log p(D'_{m}\mid C_j)
\right|
\label{eq:d1nb_weighted}
\end{equation}

%où $w_m$ est le poids de la variable $m$, identique pour toutes les classes cibles. Chaque instance est alors représentée par un vecteur de dimension $MJ$, dont les composantes, indexées par les couples $(m,j)$, sont pondérées par le poids $w_m$ associé à la variable $m$. La distance (Équation~\ref{majoration_nb}), où $d_{SNB}^{1}$ est remplacée par $d_{SNB}^{1}$, constitue alors une version pondérée de la distance supervisée initiale, dans laquelle l'influence relative de chaque variable est modulée par les poids issus du classifieur.

where $w_m$ is the weight of variable $m$, which is the same for all target classes. Each instance is then represented by an $MJ$-dimensional vector, whose components—indexed by the pairs $(m,j)$—are weighted by the weight $w_m$ associated with variable $m$. The distance (Equation~\ref{majoration_nb}), where $d_{SNB}^{1}$ is replaced by $d_{SNB}^{1}$, is thus a weighted version of the initial supervised distance, in which the relative influence of each variable is modulated by the weights derived from the classifier.

%\section{Proposition d'une distance supervisée discriminante fondée sur les log-odds}
\section{Proposal for a log-odds-based supervised discriminant distance}
\label{logodds}

\subsection{Motivation}

%Dans \cite{Lemaire2015}, les auteurs définissent une représentation supervisée induite par le classifieur naïf de Bayes, telle que présentée dans la section précédente. La majoration donnée à l'équation~\ref{majoration_nb} relie la proximité de deux observations au sens de cette représentation aux écarts entre leurs log-probabilités a posteriori. Elle fait toutefois intervenir, outre la distance $d_1^{\mathrm{SNB}}(D,D')$, un terme de normalisation globale, $ J\left|\log p(D)-\log p(D')\right|.$ Ce terme n'est pas indépendant des log-vraisemblances conditionnelles, puisque $p(D)$ est déterminée à partir des scores du modèle. L'équation~\ref{majoration_nb} ne permet pas d'attribuer totalement la proximité prédictive des observations à la seule similarité de leurs contributions discriminantes. Cela motive la construction d'une nouvelle représentation, non plus à partir des log-vraisemblances conditionnelles prises séparément, mais à partir de quantités directement liées aux contrastes entre classes, à savoir les log-odds.

In \cite{Lemaire2015}, the authors define a supervised representation induced by the naive Bayes classifier, as presented in the previous section. The upper bound given in equation~\ref{majoration_nb} relates the proximity of two observations in this representation to the differences between their log-posterior probabilities. However, in addition to the distance $d_1^{\mathrm{SNB}}(D,D')$, it involves a global normalization term, $ J|\log p(D)-$ $\log p(D')|.$ This term is not independent of the conditional log-likelihoods, since $p(D)$ is determined from the model’s scores. Equation~\ref{majoration_nb} does not allow us to attribute the predictive proximity of the observations entirely to the similarity of their discriminant contributions alone. This motivates the construction of a new representation, based no longer on the conditional log-likelihoods taken separately, but on quantities directly related to the contrasts between classes, namely the log-odds.

%%%%%%%%%%%%%%%%%%%%

%\subsection{Nouvelle représentation supervisée fondée sur les log-odds}
\subsection{A new supervised representation based on log-odds.}

% L'idée consiste à remplacer la représentation
% $
% \left(\log p(D_{m}\mid C_j)\right)_{m,j}
% $
% par une représentation fondée sur les contrastes discriminants entre une classe et l'ensemble des autres.
% %
% Soient $C=\{C_1,\dots,C_J\}$ les $J$ classes cibles. Pour chaque classe $C_j$, on définit la classe complémentaire 
% $$
% \overline{C_j}=C\setminus\{C_j\}.
% $$

The idea is to replace the representation
$
\left(\log p(D_{m}\mid C_j)\right)_{m,j}
$
with a representation based on the discriminant contrasts between a class and the set of all other classes.
Let $C=\{C_1,\dots,C_J\}$ be the $J$ target classes. For each class $C_j$, we define the complementary class $\overline{C_j}=C\setminus\{C_j\}.$

%Pour une observation $D=(D_m)_{m=1,\dots,M}$, on définit pour chaque variable $m\in\{1,\dots,M\}$ et chaque classe $j\in\{1,\dots,J\}$, la contribution locale brute au log-odds one-vs-rest\footnote{Note - Ici une hypothèse de modélisation additionnelle est faite : les variables sont conditionnellement indépendantes sachant la méta-classe ($\overline{C_j})$}. par

For an observation $D=(D_m)_{m=1,\dots,M}$, we define, for each variable $m\in\{1,\dots,M\}$ and each class $j\in\{1,\dots,J\}$, the `grow' local contribution to the one-vs-rest log-odds\footnote{Note—Here, an additional modeling assumption is made: the variables are conditionally independent given the metaclass ($\overline{C_j}$)}. by
\begin{equation}
\ell_m^{(j)}(D)=w_m \log \frac{p(D_{m}\mid C_j)}{p(D_{m}\mid \overline{C_j})},
\label{eq:ovr_log_odds_local}
\end{equation}
where
\begin{equation}
p(D_{m}\mid \overline{C_j})
=
\sum_{r\neq j} p(D_{m}\mid C_r)\,p(C_r\mid \overline{C_j}).
\label{eq:ovr_mixture}
\end{equation}

%Cette quantité mesure la contribution de la variable $m$ à la discrimination de la classe $C_j$ contre l'ensemble des autres classes. On définit alors la représentation supervisée discriminante one-vs-rest par

This quantity measures the contribution of variable $m$ to discriminating between class $C_j$ and all other classes. The one-vs-rest supervised discriminant representation is then defined as follows:
\begin{equation}
\Phi(D)
=
\bigl(\ell_m^{(j)}(D)\bigr)_{\substack{m=1,\dots,M\\ j=1,\dots,J}}
\in \mathbb{R}^{M \times J}.
\label{eq:phi_ovr}
\end{equation}

%Autrement dit, chaque observation est représentée par l'ensemble de ses contributions locales aux $J$ problèmes binaires ``une classe contre le reste''. Cette représentation conserve donc, pour chaque variable, une information discriminante détaillée classe par classe. Une distance naturelle associée à cette représentation est la distance $\ell_1$ :

In other words, each observation is represented by the set of its local contributions to the set of $J$ binary “one class versus the rest” problems. This representation therefore preserves, for each variable, detailed discriminating information on a class-by-class basis. A natural distance associated with this representation is the $\ell_1$ distance:
\begin{equation}
d^{\mathrm{1}}_{\mathrm{odds}}(D,D')
=
\sum_{j=1}^J \sum_{m=1}^M
\left|
\ell_m^{(j)}(D)-\ell_m^{(j)}(D')
\right|.
\label{eq:d_ovr_odds}
\end{equation}

%Cette distance compare deux observations à partir de l'ensemble de leurs profils explicatifs discriminants. Deux individus seront donc proches si, pour chaque classe $C_j$, les variables explicatives contribuent de manière similaire au contraste entre $C_j$ et les autres classes.

This distance compares two observations based on their sets of discriminating explanatory profiles. Two individuals will therefore be considered similar if, for each class $C_j$, the explanatory variables contribute similarly to the contrast between $C_j$ and the other classes.

%Cette formulation présente plusieurs avantages. Premièrement, elle est directement liée au caractère discriminant de la décision, puisque chaque composante mesure une contribution au log-odds d'un problème binaire one-vs-rest. Deuxièmement, elle conserve l'information directionnelle portée par chaque classe, contrairement à une agrégation préalable des contributions. Elle élimine naturellement le terme de normalisation $p(D)$ de l'équation \ref{majoration_nb}. On a à présent la majoration suivante\footnote{C'est une garantie unidirectionnelle. La preuve de cette majoration est donnée en Annexe %\cite{suplementarymaterial-github}.}:

This formulation has several advantages. Firstly, it is directly related to the discriminative nature of the decision, since each component measures a contribution to the log-odds of a binary “one-vs-rest” problem. Secondly, unlike a prior aggregation of the contributions, it preserves the directional information carried by each class. It naturally eliminates the normalization term $p(D)$ from equation \ref{majoration_nb}. We now have the following bound\footnote{This is a unidirectional guarantee. The proof of this bound is given in the Appendix.}
\begin{align}
\label{majoration_nb_2}
\Delta^1(D,D') \leq  d_{odds}^{1}(D,D')  %\leq  d_{NB}^{1}(D,D') 
%\mbox{avec }  p(D)= \prod_m p(D_{mi_m}) = \prod_m  \sum_j p(D_{mi_m}|C_j) p(C_j) \nonumber
\end{align}

\subsection{Discussion}

% La distance initialement proposée dans \cite{Lemaire2015} est cohérente avec la structure additive du classifieur naïf de Bayes, mais elle reste fondée sur les log-vraisemblances conditionnelles prises séparément. Elle ne capture donc pas de manière optimale le caractère discriminant du modèle. La distance fondée sur les log-odds présente plusieurs avantages :
% \begin{itemize}
%     \item elle élimine exactement le terme de normalisation $\log p(D)$ ;
%     \item elle est plus directement liée à la règle de décision ;
%     \item elle permet de regrouper des individus ayant des profils prédictifs similaires, même lorsque leurs vraisemblances brutes diffèrent.
% \end{itemize}

The distance initially proposed in \cite{Lemaire2015} is consistent with the additive structure of the naive Bayes classifier. However, it is still based on conditional log-likelihoods considered separately. Consequently, it does not optimally capture the model’s discriminative power. The log-odds-based distance has several advantages:
\begin{itemize}
    \item it exactly eliminates the normalization term $\log p(D)$;
    \item it is more directly related to the decision rule;
    \item it allows for the grouping of individuals with similar predictive profiles, even when their `grow' likelihoods differ.
\end{itemize}

%Elle semble donc mieux adaptée à l'objectif de segmentation supervisée poursuivi, en particulier lorsque l'on souhaite construire des groupes d'individus homogènes du point de vue de leur comportement vis-à-vis de la classe cible. L'analyse proposée ici montre qu'il est possible d'améliorer la distance supervisée proposée dans  \cite{Lemaire2015} en utilisant une représentation fondée sur les log-odds entre classes. Cette nouvelle représentation reste compatible avec la structure du classifieur naïf de Bayes, mais présente un meilleur alignement avec l'objectif discriminant du problème. Dans le cadre d'un clustering supervisé orienté vers la proximité des scores tel que suggéré dans \cite{Lemaire2015}, cette distance constitue donc une alternative  à la distance définie à partir des seules log-vraisemblances conditionnelles.

It therefore appears better suited to the supervised segmentation objective at hand, particularly when the goal is to form groups of individuals who are homogeneous in terms of their behavior with respect to the target class. The analysis presented here shows that it is possible to improve the supervised distance proposed in \cite{Lemaire2015} by using a representation based on the log-odds between classes. This new representation remains compatible with the structure of the naive Bayes classifier but is better aligned with the problem’s discriminative objective. In the context of supervised clustering focused on score proximity, as suggested in \cite{Lemaire2015}, this distance therefore constitutes an alternative to the distance defined solely based on conditional log-likelihoods.

%Les distances considérées dans cet article sont motivées non seulement par leur interprétation représentationnelle, mais aussi par leur capacité à contrôler des écarts entre sorties du classifieur. Chaque géométrie induite est reliée à une notion de proximité pertinente du point de vue prédictif. La distance supervisée initiale $d^{SNB}_1$ est justifiée par le fait qu’elle contrôle directement les écarts entre les log-probabilités a posteriori des classes. L'équation~\ref{majoration_nb} majore ces écarts par la somme de $d_1^{\mathrm{SNB}}(D,D')$ et d'un terme additionnel associé à la normalisation globale. La proximité selon $d_1^{\mathrm{SNB}}$ seule ne suffit donc pas, sur la base de cette inégalité, à garantir une faible valeur de $\Delta_1(D,D')$ si les termes de normalisation diffèrent fortement. La distance discriminante $d^{odds}_1$ bénéficie d’une justification prédictive plus forte, puisqu’elle borne directement l’écart $\Delta_1(D,D')$ entre vecteurs de scores logarithmiques de classes, sans faire intervenir le terme de normalisation global. 

The distances considered in this article are motivated not only by their representational interpretation, but also by their ability to control discrepancies between classifier outputs. Each induced geometry is linked to a notion of proximity that is relevant from a predictive standpoint. The initial supervised distance $d^{SNB}_1$ is justified by the fact that it directly controls the differences between the posterior log-probabilities of the classes. Equation~\ref{majoration_nb} bounds these differences by the sum of $d_1^{\mathrm{SNB}}(D,D')$ and an additional term associated with global normalization. Proximity based on $d_1^{\mathrm{SNB}}$ alone is therefore not sufficient, based on this inequality, to guarantee a small value of $\Delta_1(D,D')$ if the normalization terms differ significantly. The discriminant distance $d^{odds}_1$ has a stronger predictive justification, since it directly bounds the difference $\Delta_1(D,D')$ between vectors of logarithmic class scores, without involving the global normalization term.

%%%%%%%%%%%%%%%%%%%%%%%

%\section{Relation avec les valeurs de Shapley}
\section{Link to Shapley Values}
\label{sec-shap}

%Cette section établit un lien formel entre la représentation supervisée fondée sur les contributions aux log-odds et la représentation construite à partir des valeurs de Shapley analytiques du classifieur naïf de Bayes  \cite{LemaireEtAl2023}. Nous montrons que, dans le cas binaire et sous le choix de distribution de référence considéré pour les valeurs de Shapley, ces deux représentations induisent exactement la même géométrie inter-individus : elles ne diffèrent que par une translation fixe, indépendante de l’individu. En particulier, la distance ($\ell_1$) entre vecteurs de contributions aux log-odds est exactement égale à la distance ($\ell_1$) entre vecteurs de valeurs de Shapley analytiques.

This section establishes a formal link between supervised representation based on log-odds contributions and representation constructed from the analytical Shapley values of the naive Bayes classifier  \cite{LemaireEtAl2023}. We show that, in the binary case and under the chosen reference distribution for the Shapley values, these two representations induce exactly the same inter-individual geometry: they differ only by a fixed translation, independent of the individual. In particular, the ($\ell_1$) distance between vectors of log-odds contributions is exactly equal to the ($\ell_1$) distance between vectors of analytical Shapley values

%Les valeurs de Shapley ont été introduites en théorie des jeux coopératifs afin de répartir équitablement la valeur produite par une coalition entre ses joueurs \cite{Shapley1953}. Dans le contexte de l'explicabilité des modèles, elles fournissent un cadre additif pour attribuer à chaque variable une contribution à une prédiction individuelle \cite{LundbergLee2017}. Dans le cas particulier du naïf de Bayes pondéré étudié ici, cette attribution peut être calculée analytiquement à partir de la structure additive du modèle \cite{LemaireEtAl2023}.

Shapley values were introduced in cooperative game theory to distribute the value generated by a coalition among its players fairly \cite{Shapley1953}. In the context of model explainability, they provide an additive framework for attributing the contribution of each variable to an individual prediction \cite{LundbergLee2017}. In the specific case of the weighted Naive Bayes model studied here, this contribution can be calculated analytically based on the model’s additive structure \cite{LemaireEtAl2023}.

\begin{theorem}[Equivalence Between Log-Odds Distance and Shapley Va\-lue Distance]
\label{thm:equiv_shap_logodds}

%Considérons un classifieur naïf de Bayes pondéré dans le cas binaire, avec classes $C_1$ et $C_0$, et une observation $D=(D_1,\dots,D_M)$. On définit, pour chaque variable $m \in \{1,\dots,M\}$, la contribution locale brute au log-odds par

Consider a weighted naive Bayes classifier in the binary case, with classes
$C_1$ and $C_0$, and an observation $D=(D_1,\dots,D_M)$.
For each variable $m \in \{1,\dots,M\}$, we define the local contribution to the log-odds as
\begin{equation}
\ell_m(D)
=
w_m \log \frac{p(D_m \mid C_1)}{p(D_m \mid C_0)}
\label{eq:ellm}
\end{equation}

%On considère également la valeur de Shapley analytique proposée pour le naïf de Bayes :
We also consider the analytical Shapley value proposed for the Naive Bayes model:
\begin{equation}
\phi_m(D)
=
w_m
\left(
\log \frac{p(D_m \mid C_1)}{p(D_m \mid C_0)}
-
\mathbb{E}_{\mathcal{D}_m}\left[
\log \frac{p(D_m \mid C_1)}{p(D_m \mid C_0)}
\right]
\right)
\label{eq:shap_nb}
\end{equation}

%Alors, pour deux individus $D$ et $D'$, la distance $\ell_1$ entre vecteurs de valeurs de Shapley est exactement égale à la distance $\ell_1$  entre vecteurs de contributions aux log-odds :

Thus, for two individuals $D$ and $D'$, the $\ell_1$ distance between their Shapley value vectors is exactly equal to the $\ell_1$ distance between their log-odds contribution vectors:
\begin{equation}
\sum_{m=1}^M |\phi_m(D)-\phi_m(D')|
=
\sum_{m=1}^M |\ell_m(D)-\ell_m(D')|
\label{eq:equiv_dist}
\end{equation}

%Autrement dit, la distance supervisée fondée sur les log-odds coïncide exactement avec la distance $L^1$ entre les vecteurs de valeurs de Shapley analytiques.

In other words, the supervised distance based on log-odds coincides exactly with the $L^1$ distance between vectors of analytic Shapley values.
\end{theorem}

\begin{proof}
%Pour chaque variable $m$, posons
For each variable $m$, let
\begin{equation}
\mu_m
=
w_m\,
\mathbb{E}_{\mathcal{D}_m}\left[
\log \frac{p(D_m \mid C_1)}{p(D_m \mid C_0)}
\right]
\label{eq:mum}
\end{equation}

%D'après la définition \eqref{eq:shap_nb}, on peut réécrire la valeur de Shapley sous la forme\footnote{Où l’espérance, $\mu_m$, est prise par rapport à la distribution marginale de la variable recodée ($D_m$) dans la population de référence — estimée sur l’échantillon d’apprentissage — de sorte que $\mu_m$ est constant pour tous les individus.}

According to the definition (see \eqref{eq:shap_nb}), the Shapley value can be rewritten as\footnote{Where the expectation, $\mu_m$, is taken with respect to the marginal distribution of the recoded variable ($D_m$) in the reference population — estimated from the training sample — such that $\mu_m$ is constant for all individuals.}
\begin{equation}
\phi_m(D)=\ell_m(D)-\mu_m
\label{eq:phi_ell_mu}
\end{equation}

%En effet, en remplaçant $\ell_m(D)$ par son expression donnée en \eqref{eq:ellm}, on obtient immédiatement

In fact, substituting the expression for $\ell_m(D)$  given in equation \eqref{eq:ellm} immediately yields
$$
\phi_m(D)
=
w_m \log \frac{p(D_m \mid C_1)}{p(D_m \mid C_0)}
-
w_m\,
\mathbb{E}_{\mathcal{D}_m}\left[
\log \frac{p(D_m \mid C_1)}{p(D_m \mid C_0)}
\right]
=
\ell_m(D)-\mu_m
$$

%Considérons maintenant deux individus $D$ et $D'$. Pour chaque variable $m$, on a

Now, let us consider two individuals: $D$ and $D'$. For each variable m, we have
\begin{align}
\phi_m(D)-\phi_m(D')
&=
\bigl(\ell_m(D)-\mu_m\bigr)-\bigl(\ell_m(D')-\mu_m\bigr) \nonumber\\
&=
\ell_m(D)-\ell_m(D').
\label{eq:diff_phi_diff_ell}
\end{align}

%En prenant la valeur absolue, on a
Taking the absolute value gives us:
\begin{equation}
|\phi_m(D)-\phi_m(D')|
=
|\ell_m(D)-\ell_m(D')|.
\label{eq:abs_diff}
\end{equation}

%En sommant sur toutes les variables $m=1,\dots,M$, on obtient
Summing over all the variables, $m=1,\dots,M$, we obtain
\begin{equation}
\sum_{m=1}^M |\phi_m(D)-\phi_m(D')|
=
\sum_{m=1}^M |\ell_m(D)-\ell_m(D')|.
\end{equation}

%Ceci démontre l’égalité \eqref{eq:equiv_dist}.Les vecteurs de valeurs de Shapley analytiques et les vecteurs de contributions aux log-odds ne diffèrent que d’une translation fixe, coordonnée par coordonnée, donnée par $(\mu_m)_{m=1,\dots,M}$. Or une translation ne modifie pas les distances entre individus. Par conséquent, les deux distances\footnote{La norme ($\ell_1$) est retenue dans cet article car elle permet une lecture additive des écarts entre contributions individuelles des variables mais l’égalité démontrée ci-dessus ne dépend pas spécifiquement du choix de la norme.} coïncident exactement.

This proves the equality \eqref{eq:equiv_dist}. Analytic Shapley value vectors and log-odds contribution vectors differ only by a fixed translation, coordinate by coordinate, given by $(\mu_m)_{m=1,\dots,M}$. However, a translation does not alter the distances between individuals. Consequently, the two distances\footnote{The $\ell_1$ norm is used in this article because it allows for an additive interpretation of the differences between individual contributions of the variables, but the equality proven above does not depend specifically on the choice of norm.} coincide exactly.
\end{proof}

% \begin{remarque}
% \label{rem:equiv_scope}
% L'équivalence précédente vaut dans le cas binaire, pour les valeurs de Shapley analytiques du classifieur naïf de Bayes pondéré, et pour une distance définie comme une norme $L^1$ entre vecteurs de contributions. Elle ne signifie pas que toute distance construite à partir des valeurs de Shapley coïncide automatiquement avec une distance sur les log-odds. En particulier, si l'on change de norme (par exemple $L^2$) ou si l'on agrège les valeurs de Shapley autrement, l'égalité n'est plus garantie.
% \end{remarque}

%\section{Extension au cas multiclasse}

\section{Extension to the multiclass case}
\label{sec-multi}

%Dans le cas multiclasse, il convient de ne pas retenir l’agrégation proposée dans l’article sur les valeurs de Shapley \cite{LemaireEtAl2023}, qui consiste à sommer, pour chaque variable, les valeurs absolues obtenues sur les différents problèmes binaires \emph{one-vs-rest}. En effet, une telle somme produit un scalaire par variable et fait perdre l’information directionnelle portée par chaque classe. Or l’équivalence avec la distance supervisée fondée sur les log-odds ne reste exacte que si l’on conserve une représentation vectorielle détaillée, indexée à la fois par la variable et par la classe de référence. Autrement dit, pour préserver cette équivalence, il ne faut pas sommer les contributions de Shapley entre classes avant le calcul de la distance ; il faut au contraire comparer les individus composante par composante, puis agréger les écarts obtenus sur l’ensemble des variables et des classes au moyen de la norme $\ell_1$.

In the multiclass case, the aggregation method proposed in the article on Shapley values \cite{LemaireEtAl2023}, which involves summing the absolute values obtained from the various binary \emph{one-vs-rest} problems for each variable, should not be used. Such a sum yields a scalar per variable, causing the loss of directional information carried by each class. However, equivalence with the supervised distance based on log-odds remains exact only if a detailed vector representation indexed by both the variable and the reference class is preserved. In other words, to preserve this equivalence, the Shapley contributions must not be summed across classes before calculating the distance. Instead, individuals must be compared component by component, and the differences obtained across all variables and classes must be aggregated using the $\ell_1$ norm.

%\begin{proposition}[Cas multiclasse \emph{one-vs-rest} : équivalence entre distance sur les log-odds et distance sur les valeurs de Shapley]

\begin{proposition}[Multiclass \emph{one-vs-rest} cases: equivalence between the log-odds distance and the Shapley distance]

\label{prop:ovr_shap_logodds}

%Considérons un classifieur naïf de Bayes pondéré à $J$ classes $\mathcal{C}=\{C_1,\dots,C_J\}$, et une observation $D=(D_1,\dots,D_M)$. Pour chaque classe $C_j$, on définit la classe complémentaire 

Consider a weighted naive Bayes classifier with $J$ classes
$\mathcal{C}=\{C_1,\dots,C_J\}$, and an observation $D=(D_1,\dots,D_M)$.
For each class $C_j$, we define the complementary class
$$
\overline{C_j}=\mathcal{C}\setminus\{C_j\}.
$$

%Pour chaque variable $m\in\{1,\dots,M\}$ et chaque classe $j\in\{1,\dots,J\}$, on définit la contribution locale brute au log-odds \emph{one-vs-rest} par

For each variable $m\in\{1,\dots,M\}$ and each class $j\in\{1,\dots,J\}$, we define the local contribution to the \emph{one-vs-rest} log-odds as
\begin{equation}
\ell_m^{(j)}(D)
=
w_m \log \frac{p(D_m \mid C_j)}{p(D_m \mid \overline{C_j})},
\label{eq:ell_ovr}
\end{equation}
%où
where
\begin{equation}
p(D_m \mid \overline{C_j})
=
\sum_{r\neq j} p(D_m \mid C_r)\,p(C_r \mid \overline{C_j}).
\label{eq:mix_ovr}
\end{equation}

%On définit ensuite la valeur de Shapley analytique associée à la comparaison \emph{one-vs-rest} de la classe $C_j$ contre $\overline{C_j}$ par

We then define the analytical Shapley value associated with the
\emph{one-vs-rest} comparison of class $C_j$ against $\overline{C_j}$ as
\begin{equation}
\phi_m^{(j)}(D)
=
w_m\left(
\log \frac{p(D_m \mid C_j)}{p(D_m \mid \overline{C_j})}
-
\mathbb{E}_{\widetilde{X}_m \sim p(X_m)}\left[
\log \frac{p(D_m \mid C_j)}{p(D_m \mid \overline{C_j})}
\right]
\right).
\label{eq:phi_ovr2}
\end{equation}

%Si l’on conserve toutes les composantes $\phi_m^{(j)}(D)$ dans un vecteur

If we store all the components $\phi_m^{(j)}(D)$ in a vector
\begin{equation}
\Phi_{\mathrm{Shap}}(D)
=
\bigl(\phi_m^{(j)}(D)\bigr)_{\substack{m=1,\dots,M\\ j=1,\dots,J}}
\in \mathbb{R}^{MJ},
\label{eq:vec_shap_ovr}
\end{equation}
%et si l’on définit la distance $\ell_1$ associée par
and if we define the associated $\ell_1$ norm by
\begin{equation}
d_{\mathrm{Shap}}^1(D,D')
=
\sum_{j=1}^{J}\sum_{m=1}^{M}
\left|
\phi_m^{(j)}(D)-\phi_m^{(j)}(D')
\right|,
\label{eq:d_shap_ovr}
\end{equation}
%alors cette distance coïncide exactement avec la distance \emph{one-vs-rest} fondée sur les contributions aux log-odds :
so this distance coincides exactly with the \emph{one-vs-rest} distance
based on the log-odds contributions:
\begin{equation}
d_{\mathrm{odds}}^1(D,D')
=
\sum_{j=1}^{J}\sum_{m=1}^{M}
\left|
\ell_m^{(j)}(D)-\ell_m^{(j)}(D')
\right|.
\label{eq:d_odds_ovr}
\end{equation}

%Autrement dit,
In other words
\begin{equation}
d_{\mathrm{Shap}}^1(D,D')
=
d_{\mathrm{odds}}^1(D,D').
\label{eq:egalite_ovr}
\end{equation}

\end{proposition}

\begin{remark}
\label{rem:ovr_link_article}
%{ Cette formulation est cohérente avec l’extension multiclasse proposée dans l’article sur les valeurs de Shapley  \cite{LemaireEtAl2023}, pour le naïf de Bayes, où chaque classe est successivement considérée comme classe positive et toutes les autres comme classe négative.}
This formulation is consistent with the multi-class extension proposed in the article on Shapley values  \cite{LemaireEtAl2023} for the Naive Bayes classifier, where each class is successively treated as the positive class and all others as the negative class.
\end{remark}

%\section{Expérimentations}
\section{Experiments}
\label{expe}

%Cette section a pour objectif d’évaluer empiriquement la pertinence des différentes distances introduites précédemment pour la structuration d’un voisinage local cohérent avec le comportement prédictif du classifieur. Pour cela, nous utilisons ces distances dans le cadre d'un classifieur des $k$ plus proches voisins (k-NN), classification fondée sur la structure locale du voisinage \cite{CoverHart1967}, pour plusieurs valeurs de $k$. Ce protocole permet d'évaluer directement la qualité discriminante des géométries induites par les différentes distances.

The purpose of this section is to empirically evaluate the relevance of the various distances introduced earlier for structuring a local neighborhood  consistent with the classifier’s predictive behavior. To achieve this, we use these distances within the framework of a $k$-nearest-neighbor (k-NN) classifier—a classification method based on the local neighborhood structure \cite{CoverHart1967}—for several values of $k$. This protocol allows us to directly evaluate the discriminative quality of the geometries induced by the different distances directly.

%Nous comparons trois distances. La première est la distance supervisée pondérée initiale définie à l’équation~\ref{eq:d1nb_weighted}, fondée sur les log-vraisemblances conditionnelles prises séparément mais avec le terme de normalisation global:

We compare three distances. The first is the initial weighted supervised distance defined by equation~\ref{eq:d1nb_weighted}, based on the conditional log-likelihoods taken separately but with the global normalization term:
\begin{equation}
d_{global}^{1}(D,D')=
\left[ d_{SNB}^{1}(D,D') + J \left| log(p_w(D))-log(p_w(D')) \right| \right]
\label{eq:dglobal}
\end{equation}

%La deuxième est la distance discriminante fondée sur les contributions locales aux log-odds \emph{one-vs-rest}, définie à l’équation~\ref{eq:egalite_ovr} :

The second is the discriminant distance based on the local contributions to the \emph{one-vs-rest} log-odds, defined by equation~\ref{eq:egalite_ovr}:
\begin{equation}
d_{odds}^{1}(D,D')=d_{shap}^{1}(D,D')=
\sum_{j=1}^{J}\sum_{m=1}^{M}
w_m \left|
\log \frac{p(D_m \mid C_j)}{p(D_m \mid \overline{C_j})}- \log \frac{p(D'_m \mid C_j)}{p(D'_m \mid \overline{C_j})} \right|
\label{eq:dshap}
\end{equation}

%Enfin, afin d'évaluer l'effet du terme de normalisation global, nous considérons la distance obtenue en retirant de l'équation~(\ref{eq:dglobal}) le terme $J\left|\log p_w(D)-\log p_w(D')\right|$. On retrouve alors exactement la distance supervisée pondérée définie à l'équation~(\ref{eq:d1nb_weighted}), $d_{SNB}^{1}(D,D')$, utilisée dans \cite{Lemaire2015}:

Finally, to evaluate the effect of the global normalization term,
we consider the distance obtained by removing from equation~(\ref{eq:dglobal})
the term $J |\log p_w(D) -$ $ \log p_w(D') |$. We then obtain
exactly the weighted supervised distance defined in equation~(\ref{eq:d1nb_weighted}), $d_{SNB}^{1}(D,D')$, used in \cite{Lemaire2015}:
\begin{equation}
d_{SNB}^{1}(D,D')=\sum_{m=1}^{M}\sum_{j=1}^{J} w_m \left| \log p(D_m \mid C_j) - \log p(D'_m \mid C_j) \right|.
\end{equation}

%\subsection{Jeux de données} 
\subsection{Datasets} 

%Dans ce qui suit la comparaison des algorithmes a été réalisée sur douze jeux de données différents issus du dépôt UCI \cite{Lichman2013}. Ces jeux de données ont été sélectionnés pour leur diversité en termes de nombre de classes, de variables et de nombre d'instances (voir Tableau~(\ref{used_data})\footnote{Le code fourni sur Github permet de rejouer l'ensemble des expériences sur d'autres jeux de données.}).

In what follows, the comparison of algorithms was performed on twelve different datasets from the UCI repository \cite{Lichman2013}. These datasets were selected for their diversity in terms of the number of classes, variables, and instances (see Table~(\ref{used_data})\footnote{The code provided on GitHub allows you to reproduce all of the experiments on other datasets.}).

\begin{table}
%\fontsize{8}{10}\selectfont
\centering
%\caption{{\small Jeux de données utilisés dans les expérimentations. $V_n$ désigne le nombre de variables numériques, $V_c$ le nombre de variables catégorielles, et $C$ le nombre de classes.}}
\caption{Datasets used in the experiments.
$V_n$ denotes the number of numerical variables,
$V_c$ the number of categorical variables, and $C$ the number of classes.}
\label{used_data}
\begin{tabular}{lrrrr||lrrrr}
\toprule
%\textbf{Données} & \# exemples & \textbf{$V_n$} & \textbf{$V_c$} & \textbf{$C$} & \textbf{Données} & \# exemples & \textbf{$V_n$} & \textbf{$V_c$} & \textbf{$C$}\\
\textbf{Dataset} & \# examples & \textbf{$V_n$} & \textbf{$V_c$} & \textbf{$C$} & \textbf{Dataset} & \# examples & \textbf{$V_n$} & \textbf{$V_c$} & \textbf{$C$}\\
\midrule
Adult           & 48\,842 & 7  & 8  & 2  & Pima            & 768    & 8  & 0  & 2  \\
Heart           & 270    & 10 & 3  & 2  & Segmentation    & 2\,310 & 19 & 0  & 7  \\
Letter          & 20\,000 & 16 & 0  & 26 & Shuttle         & 58\,000 & 9  & 0  & 7  \\
Mushroom        & 8\,416 & 0  & 22 & 2  & SickEuthyroid   & 3\,163 & 7  & 18 & 2  \\
Pendigits       & 10\,992 & 16 & 0  & 10 & Vehicle         & 846    & 18 & 0  & 4  \\
Phoneme         & 5\,404 & 5  & 0  & 2  & Waveform        & 5\,000 & 40 & 0  & 3  \\
\bottomrule
\end{tabular}
\end{table}

%\subsection{Implémentation des distances} 
\subsection{Implementation of distances}

%Les expérimentations ont été réalisées avec la librairie \texttt{Khiops} \cite{khiops}, qui met en œuvre une chaîne de préparation supervisée et d’apprentissage adaptée aux classifieurs probabilistes. Les variables numériques sont discrétisées automatiquement selon une approche supervisée fondée sur le critère MODL (\emph{Minimum Description Length}), qui recherche une partition en intervalles offrant un compromis bayésien entre précision descriptive et complexité du modèle \cite{Boulle2004Discretization}. Les variables catégorielles sont traitées par un groupement supervisé de modalités, également fondé sur le critère MODL : les modalités dont les distributions conditionnelles de classes sont proches peuvent être réunies dans un même groupe \cite{Boulle2005Grouping}. Le classifieur utilisé (et qui donne les poids des variables) est le fractional naïf de Bayes sélectif pondéré de Khiops tel que décrit dans \cite{HueBoulle2024FNB}. Enfin, les valeurs de Shapley sont calculées analytiquement à partir de la structure additive du naïf de Bayes pondéré tel que décrit dans \cite{LemaireEtAl2023} et disponible dans la librairie. L'ensemble des expériences est reproductible à l'aide du code disponible dans le GitHub \cite{suplementarymaterial-github}.

The experiments were conducted using the \texttt{Khiops} library \cite{khiops}, which implements a supervised preprocessing and learning pipeline tailored for probabilistic classifiers. Numeric variables were automatically discretized using a supervised approach based on the MODL (\emph{Minimum Description Length}) criterion, which seeks a partition into intervals that provides a Bayesian trade-off between descriptive accuracy and model complexity \cite{Boulle2004Discretization}. Categorical variables are processed using a supervised grouping of modalities, also based on the MODL criterion: modalities with similar conditional class distributions can be grouped together \cite{Boulle2005Grouping}. The classifier used (which assigns weights to the variables) is the Khiops weighted selective fractional Naive Bayes classifier, as described in \cite{HueBoulle2024FNB}. Finally, Shapley values are calculated analytically based on the additive structure of the weighted naive Bayes classifier, as described in \cite{LemaireEtAl2023} and implemented in the library. All experiments can be reproduced using the code available on GitHub \cite{suplementarymaterial-github}.

%\subsection{Résultats expérimentaux}
\subsection{Experimental Results}

{\bf 5.3.1 k-nn Classification performances}

%Le Tableau~(\ref{tab:results}) compare les performances du classifieur 1-ppv obtenu avec les trois distances supervisées considérées. Les premiers résultats sont présentés pour k=1 afin d’isoler au maximum l’effet propre de la géométrie induite par chacune des distances. Avec un seul plus proche voisin, la prédiction dépend directement de l’observation la plus proche dans l’espace supervisé considéré, sans effet de lissage ou de vote majoritaire entre plusieurs voisins. Ce choix permet donc d’évaluer de manière particulièrement directe la capacité des distances à rapprocher des individus ayant un comportement prédictif similaire. 

\begin{table}[!htbp]
%\fontsize{8}{9}\selectfont
\centering
\caption{Performance moyenne en Test sur 10 folds stratifiés par dataset}
\label{tab:results}
\footnotesize
\begin{tabular}{lcccccc}
\toprule
%\textbf{Dataset} &
%\multicolumn{2}{c}{\textbf{Équation 27}} &
%\multicolumn{2}{c}{\textbf{Équation 28}} &
%\multicolumn{2}{c}{\textbf{Équation 29}} \\
\textbf{Dataset} &
\multicolumn{2}{c}{$d_{global}^{1}(D,D')$} &
\multicolumn{2}{c}{$d_{shap}^{1}(D,D')$} &
\multicolumn{2}{c}{$d_{SNB}^{1}(D,D')$} \\
\cmidrule(lr){2-3}
\cmidrule(lr){4-5}
\cmidrule(lr){6-7}
& ACC & AUC & ACC & AUC & ACC & AUC \\
\midrule
Adult	&0,8239	&0,7541	&0,8241	&0,7547	&0,8256	&0,7557 \\
Heart	&0,7889	&0,7867	&0,8148	&0,8117	&0,8148	&0,8117 \\
Letter	&0,9348	&0,9659	&0,9398	&0,9686	&0,9431	&0,9703 \\
Mushroom	&1,0000	&1,0000	&1,0000	&1,0000	&1,0000	&1,0000 \\
Pendigits	&0,9731	&0,9852	&0,9773	&0,9874	&0,9758	&0,9867 \\
Phoneme	&0,8941	&0,9297	&0,8915	&0,9277	&0,8900	&0,9271 \\
Pima	&0,6756	&0,6486	&0,6730	&0,6545	&0,6717	&0,6455 \\
Segmentation	&0,9732	&0,9843	&0,9766	&0,9864	&0,9749	&0,9854 \\
Shuttle	&0,9995	    &0,9425	&0,9996	&0,9507	&0,9996	&0,9507 \\
SickEuthyroid	    &0,3639	&0,6433	&0,4145	&0,6666	&0,3636	&0,6432 \\
Vehicle	&0,6833	    &0,7898	&0,6915	&0,7950	&0,6892	&0,7937 \\
Waveform	        &0,7596	&0,8196	&0,7720	&0,8289	&0,7680	&0,8259 \\ 
\midrule
Moyenne	&0,8225	&0,8541	& \textbf{0,8312}	& \textbf{0,8610}	&0,8264	&0,8580 \\
\midrule
\end{tabular}
%\vspace{-4mm}
\end{table}

Table~(\ref{tab:results}) compares the performance of the 1-ppv classifier obtained using the three supervised distances under consideration. The initial results are presented for k=1 in order to isolate as much as possible the specific effect of the geometry induced by each of the distances. With a single nearest neighbor, the prediction depends directly on the closest observation in the supervised space under consideration, without any smoothing effect or majority voting among multiple neighbors. This choice therefore allows for a particularly direct assessment of the distances’ ability to group together individuals with similar predictive behavior.

%Sur l’ensemble des douze jeux de données, la distance $d_{shap}^{1}$, définie à l’Équation~(\ref{eq:dshap}), obtient les meilleures performances moyennes tant en ACC qu'en AUC. Ces résultats montrent que la représentation fondée sur les contributions aux log-odds permet de construire un voisinage immédiat bien aligné avec la structure discriminante du classifieur.

Across all twelve datasets, the distance $d_{shap}^{1}$, defined in Equation~(\ref{eq:dshap}), achieves the best average performance in both ACC and AUC. These results show that the representation based on log-odds contributions allows us to construct a local neighborhood that is well aligned with the discriminative structure of the classifier.

%La Figure~\ref{fig:cd} complète l'analyse des performances moyennes en comparant les rangs des trois distances sur l'AUC one-versus-rest (rangs par jeu de données, puis le test non paramétrique de Friedman et le test post-hoc de Nemenyi \cite{Demsar2006}). Contrairement aux valeurs moyennes seules, le diagramme critique permet d'évaluer la régularité relative des méthodes sur l'ensemble des jeux de données, en réduisant l'influence éventuelle de quelques jeux particulièrement favorables ou défavorables. Le diagramme place la distance $d_{shap}^{1}$ en première position, avec un rang moyen de 1.4, devant $d_{SNB}^{1}$, dont le rangmoyen est de 2, puis $d_{global}^{1}$, avec un rang moyen de 2.6.  Cependant, devant les faibles écarts de performances (ce qui ne remet pas en cause le critical diagram) dans le Tableau \ref{tab:results}, nous avons aussi réalisé des comparaisons locales deux à deux, réalisées sur les mêmes folds au moyen d’un test t corrigé pour la validation croisée \cite{NadeauBengio2003}, avec l’ajustement de Holm des comparaisons multiples  \cite{Holm1979}, qui ne permettent de distinguer aucune paire au seuil de 1 \%. Ainsi, les écarts de rang observés ne constituent pas, dans ce protocole, une évidence statistique de supériorité de ($d_1^{\mathrm{shap}}$), mais suggèrent un avantage plutôt empirique; ce qui reste intéressant.

Figure ~\ref{fig:cd} complements the analysis of average performance by comparing the ranks of the three distances on the one-versus-rest area under the curve (AUC) (ranks by dataset, followed by the non-parametric Friedman test and the Nemenyi post-hoc test, as described in \cite{Demsar2006}). Unlike mean values alone, the critical plot enables us to evaluate the consistency of methods across all datasets by minimising the influence of a few particularly favourable or unfavourable datasets. The diagram shows that $d_{shap}^{1}$ is the best method, with a mean rank of 1.4, followed by $d_{SNB}^{1}$ with a mean rank of 2 and then $d_{global}^{1}$ with a mean rank of 2.6. However, given the small performance differences in Table 1, which do not call the critical diagram into question, we also performed local pairwise comparisons on the same folds using a t-test corrected for cross-validation \cite{NadeauBengio2003} with Holm's adjustment for multiple comparisons \cite{Holm1979}. This did not allow us to distinguish any pairs at the 1\% significance level. Thus, in this protocol, the observed rank differences do not constitute statistical evidence of the superiority of $d_1^{\mathrm{shap}}$, but suggest an empirical advantage, which remains interesting.

\begin{figure}[!ht]
    \centering
    \includegraphics[width=1.0\linewidth]{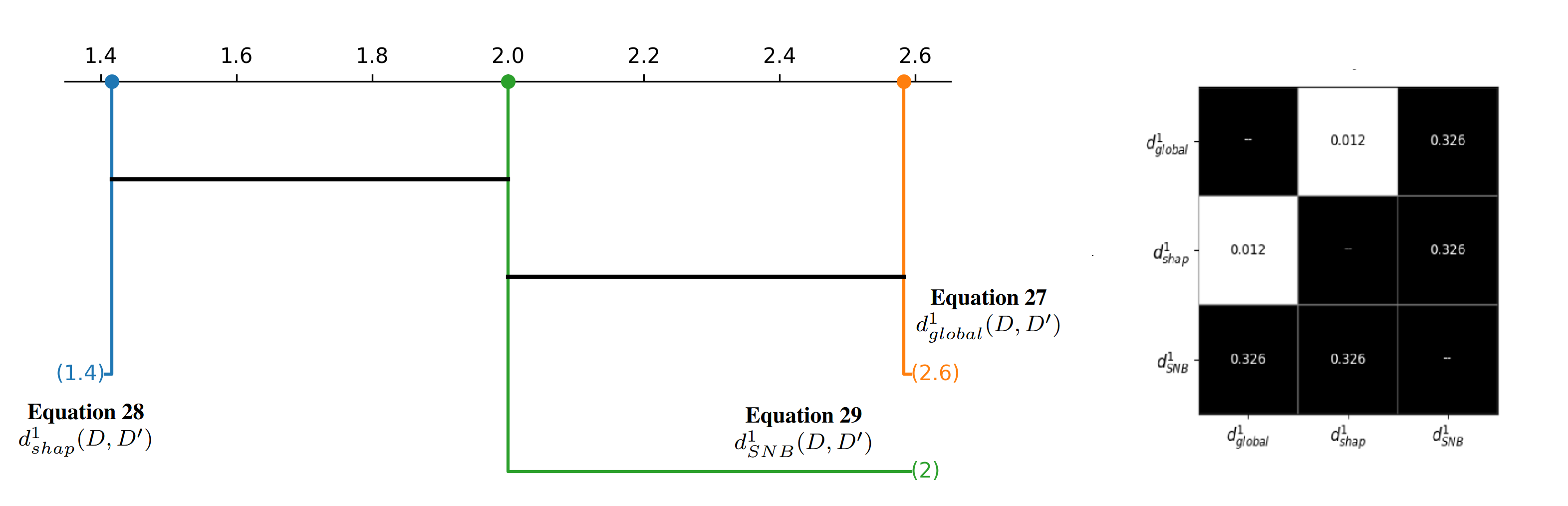}
    %\caption{{\small Critical Diagram entre les 3 distances (AUC OVR)}}
    \caption{Critical Diagram between the 3 distances (AUC OVR).}
    \label{fig:cd}
    %\vspace{-4mm}
\end{figure}

%En complément des performances de classification du $1$-NN, nous évaluons également la cohérence prédictive locale des voisinages induits par chaque distance.  Pour une observation de test $D$, soit $\widehat{c}_{\mathrm{SNB}}(D)=\arg\max_j P(C_j\mid D)$ la classe prédite par le naïf de Bayes pondéré appris sur le fold d'apprentissage. Nous définissons alors

{\bf 5.3.2 Predictive consistency\\}

In addition to the classification performance of the $1$-NN, we also evaluate the local predictive consistency of the neighborhoods induced by each distance.  For a test observation $D$, let $\widehat{c}_{\mathrm{SNB}}(D)=\arg\max_j P(C_j\mid D)$ be the class predicted by the weighted Naive Bayes classifier trained on the training fold. We then define
$$
\mathrm{LocAgree}_k(d)
=
\frac{1}{|T|}
\sum_{D\in T}
\frac{1}{k}
\sum_{D'\in N_k^d(D)}
\mathds{1}\!\left[
\widehat{c}_{\mathrm{SNB}}(D)
=
\widehat{c}_{\mathrm{SNB}}(D')
\right].
$$
%Cette quantité correspond à la proportion moyenne de voisins partageant la même classe prédite par le SNB que l'observation considérée ; une valeur élevée est donc préférable. 

This quantity corresponds to the average proportion of neighbors that share the same class predicted by the SNB as the observation in question; a high value is therefore preferable.
\begin{figure*}[!htbp]
    \centering

    \begin{subfigure}[t]{0.49\textwidth}
        \centering
        \includegraphics[width=\linewidth]{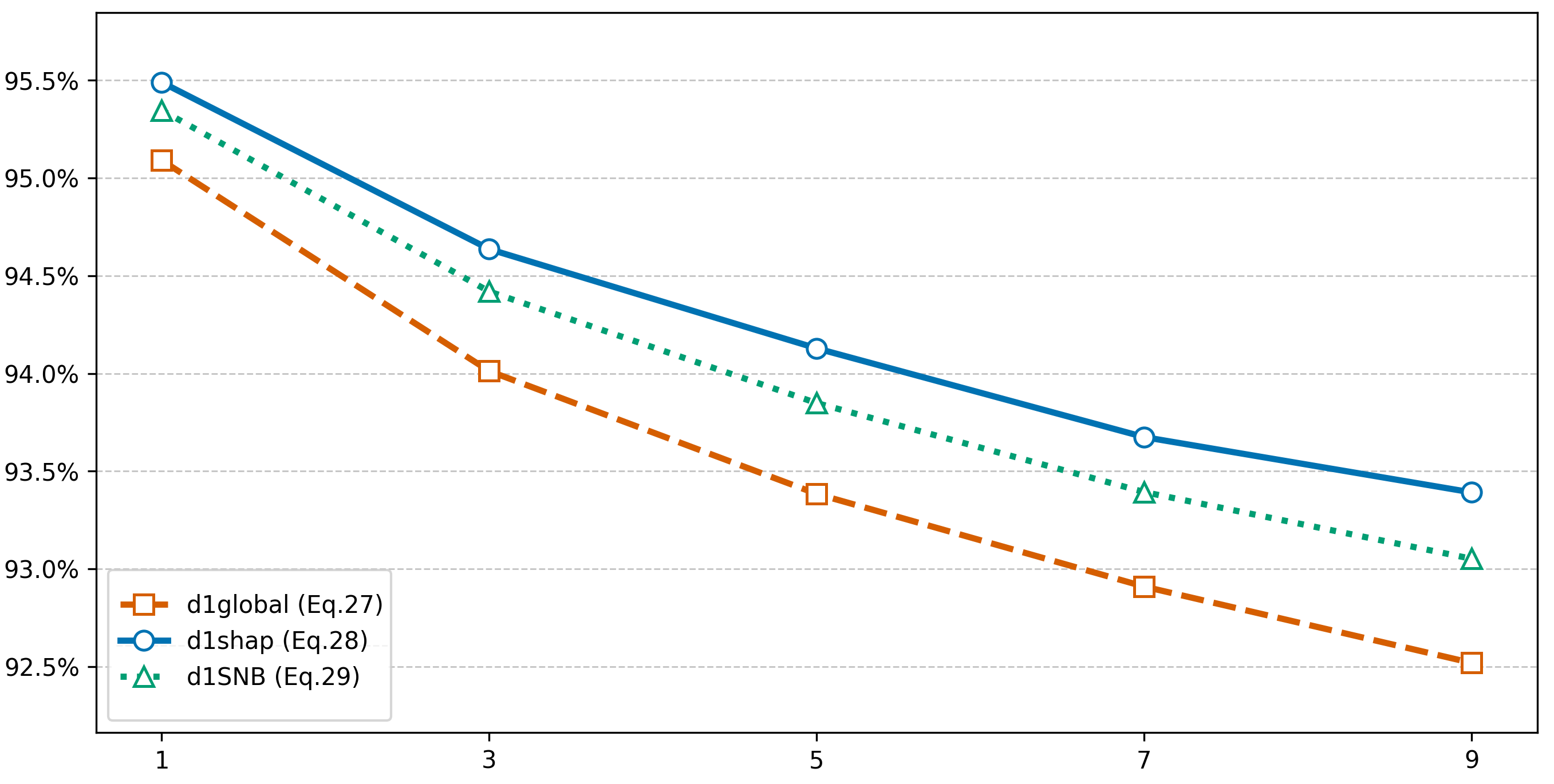}
        %\caption{Moyenne non pondérée sur les douze jeux de données.}
        \caption{Unweighted average across the twelve datasets.}
        \label{fig:local-agree-global}
    \end{subfigure}
    \hfill
    \begin{subfigure}[t]{0.49\textwidth}
        \centering
        \includegraphics[width=\linewidth]{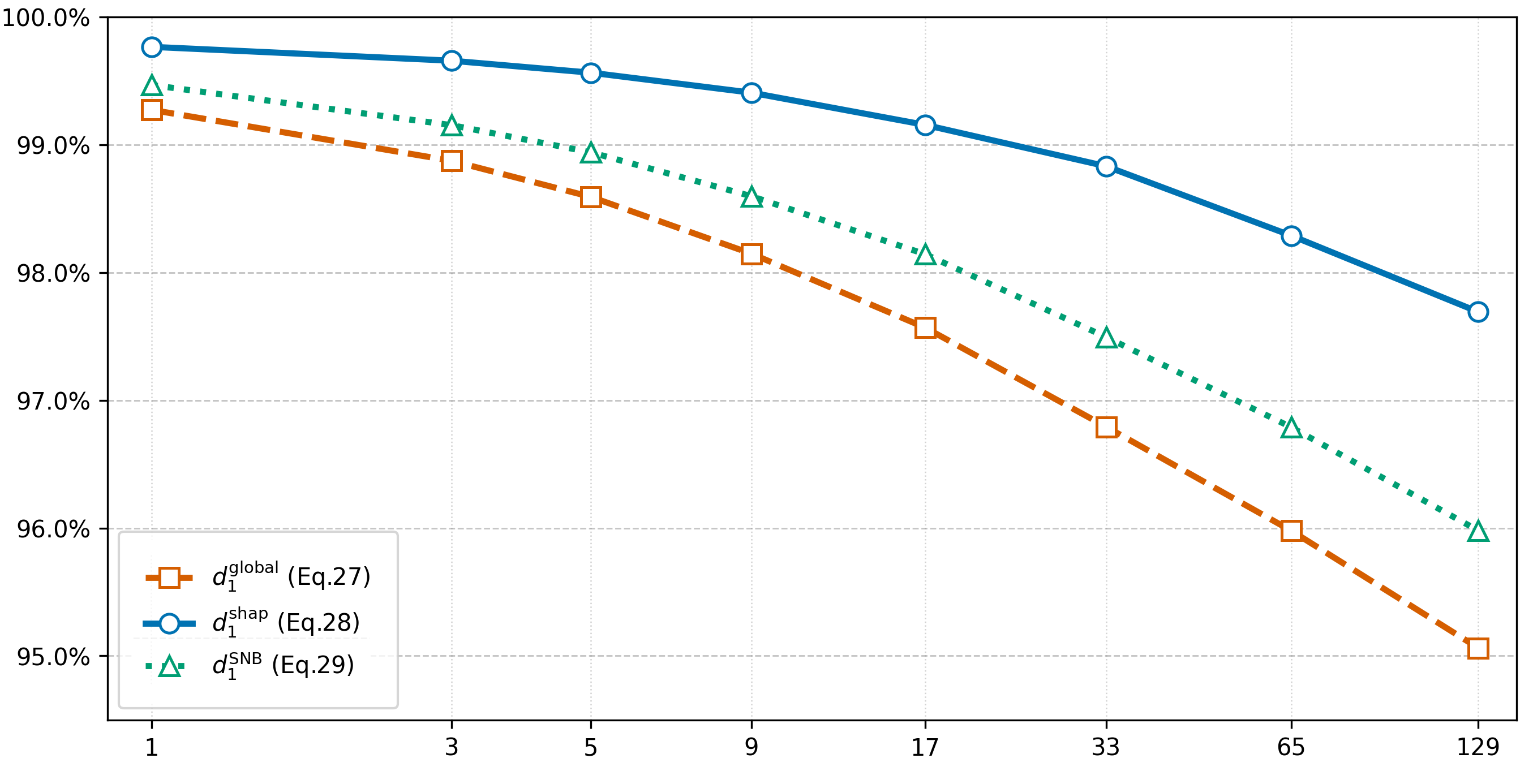}
        %\caption{Adult : moyenne sur les dix folds test, avec une abscisse en échelle logarithmique de base $2$.}
        \caption{Adult: average across the ten test folds, with the x-axis on a base-2 logarithmic scale.}
        \label{fig:local-agree-adult}
    \end{subfigure}

    %\caption{{\small Évolution de $\mathrm{LocAgree}_k$ en fonction de $k$. Une valeur élevée indique que les voisins sélectionnés par la distance considérée partagent plus fréquemment la même classe prédite par le SNB. %que l'observation de test.     La sous-figure~(\subref{fig:local-agree-global}) présente les résultats moyens sur les douze jeux de données, tandis que la sous-figure~(\subref{fig:local-agree-adult}) détaille ce comportement pour le jeu de données Adult sur un ensemble étendu de valeurs de $k$.}}
    \caption{Evolution of $\mathrm{LocAgree}_k$ as a function of $k$. A high value indicates that the neighbors selected based on the given distance more frequently share the same class predicted by the SNB than the test observation.     Subfigure~(\subref{fig:local-agree-global}) shows the average results across the twelve datasets, while subfigure~(\subref{fig:local-agree-adult}) details this behavior for the Adult dataset over a wide range of $k$ values.}
    \label{fig:local-agree-global-adult}
\end{figure*}

%La Figure~\ref{fig:local-agree-global-adult} confirme que les distances considérées induisent des voisinages globalement homogènes du point de vue de la décision du SNB, les valeurs de $\mathrm{LocAgree}_k$ restant élevées pour toutes les valeurs de $k$. La distance $d_1^{\mathrm{shap}}$ obtient les valeurs les plus élevées, ce qui indique que les voisins qu'elle sélectionne reçoivent plus fréquemment la même classe prédite que l'observation de test. Comme attendu, $\mathrm{LocAgree}_k$ tend à diminuer lorsque $k$ augmente : l'élargissement du voisinage conduit progressivement à inclure des observations plus éloignées et donc potentiellement situées de l'autre côté d'une frontière de décision. La distance construite à partir des valeurs de Shapley ne rapproche pas seulement des observations aux profils probabilistes similaires, mais tend également à préserver davantage la stabilité de la classe prédite, étant donné que $\mathrm{LocAgree}_k$ décroît moins vite pour elle que pour les autres distances.

Figure~\ref{fig:local-agree-global-adult} confirms that the considered distances result in neighborhoods that are generally homogeneous in terms of of the SNB’s decision, with the values of $\mathrm{LocAgree}_k$ remaining high for all values of $k$. The $d_1^{\mathrm{shap}}$ distance yields the highest values, indicating that the neighbors it selects are more frequently assigned the same predicted class as the test observation. As expected, $\mathrm{LocAgree}_k$ tends to decrease as $k$ increases: expanding the neighborhood gradually leads to the inclusion of more distant observations and thus those potentially located on the other side of a decision boundary. The distance constructed from the Shapley values not only brings together observations with similar probabilistic profiles but also tends to better preserve the stability of the predicted class better, since $\mathrm{LocAgree}_k$ decreases more slowly for it than for the other distances.\\

%{\bf Etude de la dispersion - }
{\bf 5.3.3 Analysis of dispersion\\}

%Pour une distance $d$ donnée et pour chaque observation de test $D$, nous notons $\mathcal{N}_k^d(D)$ l’ensemble de ses $k$ plus proches voisins dans l’échantillon d’apprentissage. Nous calculons 

For a given distance $d$ and for each test observation $D$, we denote by $\mathcal{N}_k^d(D)$ the set of its $k$ nearest neighbors in the training sample. We compute 
 \begin{equation}
 \mathrm{LocDelta}_k(d)=
 \frac{1}{|T|}
 \sum_{D\in T}
 \frac{1}{k}
 \sum_{D'\in \mathcal{N}_k^d(D)}
 \Delta_1(D,D'),
 \label{eq:locdelta}
 \end{equation}
%où $T$ désigne l’ensemble de test. Cela permet d'évaluer la cohérence prédictive locale des voisinages induits par chaque distance à l’aide d’un écart prédictif local moyen. Pour une observation test (D) et l’un de ses voisins (D'), la quantité ($\Delta_1(D,D')$) compare les vecteurs complets de log-probabilités a posteriori estimées par le FNB. Elle est nulle si et seulement si les deux observations possèdent exactement le même vecteur de probabilités prédites. Une faible valeur de ($\mathrm{LocDelta}_k$) indique donc que les voisins sélectionnés présentent des profils probabilistes proches au sens du modèle de référence. Cette mesure ne garantit toutefois pas, à elle seule, que la classe prédite soit stable ; cette propriété est examinée séparément au moyen de ($\mathrm{LocAgree}_k$).
where $T$ denotes the test set. This allows us to evaluate the local predictive consistency of the neighborhoods induced by each distance using an average local predictive deviation. For a test observation (D) and one of its neighbors (D'), the quantity ($\Delta_1(D,D')$) compares the complete vectors of posterior log-probabilities estimated by the SNB. It is zero if and only if the two observations have exactly the same vector of predicted probabilities. A low value of ($\mathrm{LocDelta}_k$) therefore indicates that the selected neighbors have similar probabilistic profiles according to the reference model. However, this measure alone does not guarantee stable predictions; this property is examined separately using ($\mathrm{LocAgree}_k$).

\begin{figure}[!h]
    \centering
    \includegraphics[width=0.55\linewidth]{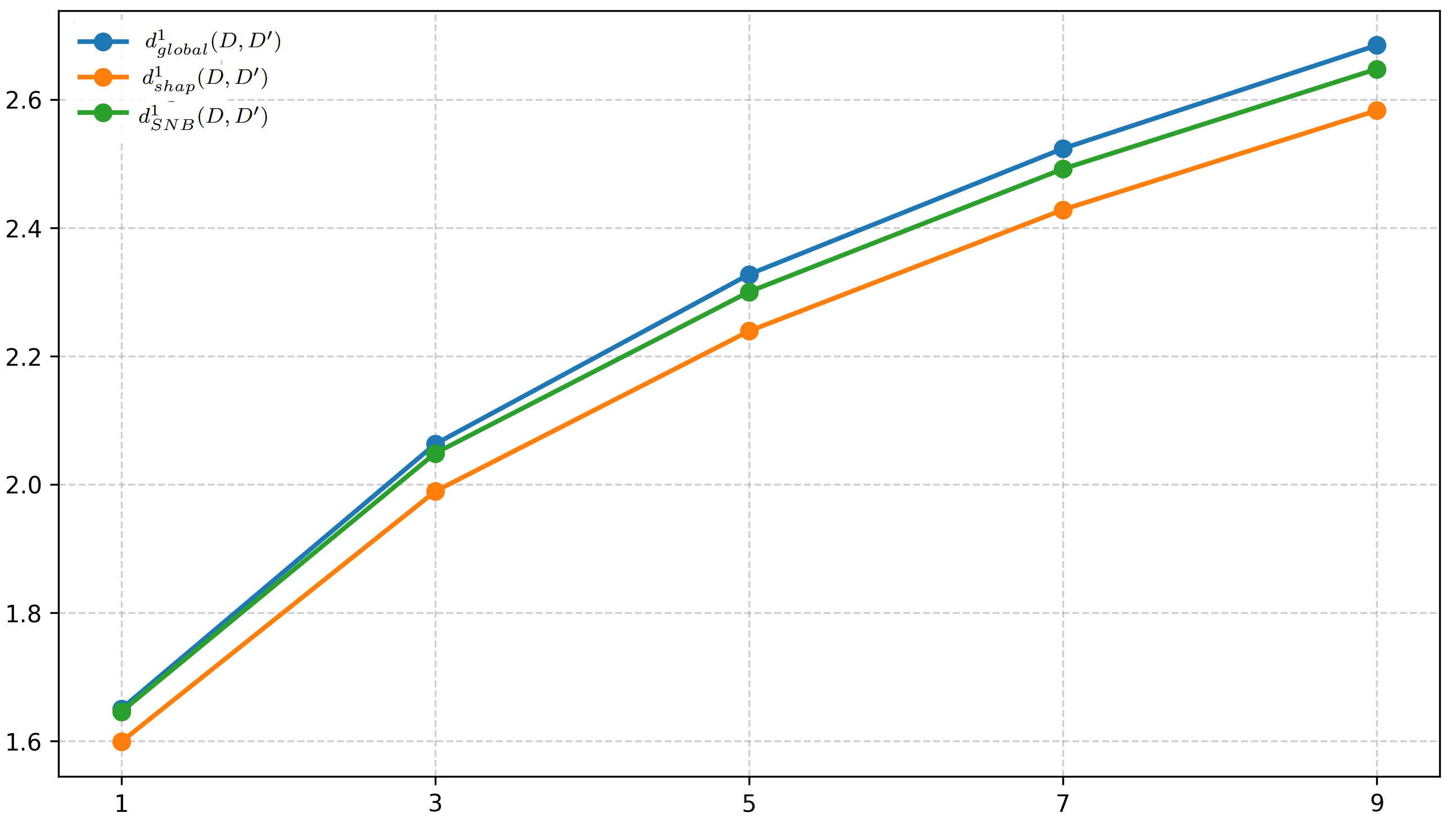}
    %\caption{Valeur des dispersion en fonction de la valeur de k (moyenne sur les 12 datasets).}
    \caption{Dispersion values as a function of k (average across the 12 datasets).}
    \label{fig:dispersion-versus-k2}
    %\vspace{-4mm}
\end{figure}

%La Figure~\ref{fig:dispersion-versus-k2} montre que la distance $d_{shap}^{1}$ induit, en moyenne, les voisinages présentant la plus faible dispersion prédictive pour toutes les valeurs de $k$ considérées. L'augmentation générale de la dispersion lorsque $k$ augmente est attendue. En élargissant le voisinage, les observations sélectionnées sont progressivement moins similaires à l'observation étudiée du point de vue de leur vecteur de scores prédictifs. L'intérêt de $d_{shap}^{1}$ est qu'elle limite plus fortement cette augmentation : à taille de voisinage égale, ses voisins ont en moyenne des log-probabilités de classes plus proches de celles de l'individu considéré.\\

Figure ~ref{fig:dispersion-versus-k2} shows that, on average, the neighbourhoods with the lowest predictive dispersion are those produced by $d_{shap}^{1}$ for all considered values of $k$. The general increase in variance as $k$ increases is to be expected. As the neighborhood expands, the selected observations become progressively less similar to the observation under study in terms of their predictive score vectors. The advantage of $d_{shap}^{1}$ is that it limits this increase more effectively: for the same neighborhood size, its neighbors have, on average, class log-probabilities closer to those of the individual under consideration.\\

%%%%%%%%%%%%%%%%%%%%%%%%%%%%%%%%
{\bf 5.3.4 Analysis of the local classification margin\\}

%Au-delà de l'accord local avec les décisions du FNB, nous évaluons la capacité de chaque distance à séparer localement les classes réelles dans le cadre du classifieur $1$-ppv. Pour chaque observation de test $D$, de classe réelle $y(D)$, nous notons $A$ l'échantillon d'apprentissage du fold considéré. Nous définissons la distance au plus proche voisin de même classe par

Beyond local agreement with the SNB’s decisions, we evaluate the
ability of each distance to locally separate the true classes within the $1$-ppv classifier framework. For each test observation $D$, of true class $y(D)$, we denote by $A$ the training sample of the fold in question. We define the distance to the nearest neighbor of the same class as
$$
d^{+}(D;d)
=
\min_{\substack{D' \in A\\ y(D')=y(D)}} d(D,D'),
$$
%et la distance au plus proche voisin de classe différente par
and the distance to the nearest neighbor of a different class by
$$
d^{-}(D;d)
=
\min_{\substack{D' \in A\\ y(D')\neq y(D)}} d(D,D').
$$

%Nous définissons alors la marge locale normalisée :
We then define the normalized local margin as follows:
$$
\mathrm{Margin}(D;d)
=
\frac{d^{-}(D;d)-d^{+}(D;d)}
     {d^{-}(D;d)+d^{+}(D;d)}.
$$

\begin{table}[!htbp]
%\fontsize{8}{9}\selectfont
\centering
%\caption{Performance moyenne en Test sur 10 folds stratifiés par dataset}
\caption{Average performance on test across 10 folds stratified by dataset}
\label{tab:margin}
\footnotesize
\begin{tabular}{lccc}
\toprule
dataset & $d_{global}^{1}(D,D')$ & $d_{shap}^{1}(D,D')$ & $d_{SNB}^{1}(D,D')$ \\
\midrule
Adult	& 0.4838 	& 0.5264 & 0.5015 \\
Heart	& 0.4626  & 0.4645 & 0.4698 \\
Letter	& 0.5931  & 0.6304 & 0.6210 \\
Mushroom & 0.9999 & 1.0000 & 1.0000 \\
Pendigits & 0.5276 & 0.5576 & 0.5443 \\
Phoneme	& 0.3073 & 0.3004 & 0.3017 \\
Pima	& 0.3057 & 0.3101 & 0.3091 \\
Segmentation & 0.7598 & 0.7948 & 0.7790 \\
Shuttle	& 0.9977 & 0.9979 & 0.9979 \\
SickEuthyroid	& 0.2703 & 0.2703 & 0.2703 \\
Vehicle	& 0.2372 & 0.2489 & 0.2550 \\
Waveform	& 0.0893 & 0.1033 & 0.0960 \\
\midrule
%Moyenne & 0.5029 & {\bf 0.5170} & 0.5121  \\
Average & 0.5029 & {\bf 0.5170} & 0.5121  \\
\midrule
\end{tabular}
\end{table}

%Cette quantité est comprise entre $-1$ et $1$. Une valeur positive signifie que le voisin le plus proche de $D$ dans l'échantillon d'apprentissage appartient à sa classe réelle. À l'inverse, une valeur négative indique qu'une observation d'une autre classe est plus proche que tout voisin de même classe. Une marge élevée traduit donc une meilleure séparation locale des classes dans la géométrie induite par la distance considérée. Les résultats sont présentés dans le tableau \ref{tab:margin}.

This quantity ranges from $-1$ to $1$. A positive value means that $D$’s nearest neighbor in the training sample belongs to its true class. Conversely, a negative value indicates that an observation from another class is closer than any neighbor from the same class. A high margin therefore reflects better local separation of classes in the geometry induced by the distance under consideration. The results are presented in Table \ref{tab:margin}.

\section{Conclusion}

%Cet article a proposé une reformulation de la distance supervisée induite par le naïf de Bayes pondéré à partir des contributions locales aux log-odds. Cette représentation est plus directement liée au caractère discriminant de la règle de décision que la représentation classique fondée sur les log-vraisemblances conditionnelles. Nous avons montré que, lorsque les contributions sont conservées classe par classe dans le cas multiclasse, la distance obtenue coïncide exactement avec la distance entre les vecteurs de valeurs de Shapley analytiques. La géométrie proposée possède ainsi une interprétation à la fois supervisée et explicative.

This article proposes a reformulation of the supervised distance induced by the weighted Naive Bayes model, which is based on local contributions to the log-odds. This representation is more closely related to the discriminative nature of the decision rule than the traditional representation based on conditional log-likelihoods. We have demonstrated that when the contributions are preserved class by class in the multiclass case, the resulting distance coincides exactly with the distance between the vectors of analytical Shapley values. The proposed geometry thus has an interpretation that is both supervised and explanatory.

%Les expérimentations sur douze jeux de données indiquent que la distance $d_1^{\mathrm{shap}}$ obtient de bonnes performances prédictives avec un classifieur $1$-ppv.  En outre, elle tend à préserver davantage l'homogénéité locale des classes prédites par le SNB, telle que mesurée par $\mathrm{LocAgree}_k$. Ces résultats suggèrent que cette distance constitue une alternative pertinente à la distance supervisée utilisée dans \citet{Lemaire2015}; ceci avec une interprétation basée sur les valeurs de Shapley. Des travaux futurs étudieront son emploi pour la segmentation supervisée et dans le cadre de la prédiction conforme.

Experiments on twelve datasets indicate that the distance $d_1^{\mathrm{shap}}$ achieves good predictive performance with a $1$-ppv classifier.  Furthermore, it tends to preserve the local homogeneity of the classes predicted by the SNB more effectively, as measured by $\mathrm{LocAgree}_k$. These results suggest that this distance is a relevant alternative to the supervised distance used in \cite{Lemaire2015}; this is achieved with an interpretation based on Shapley values. Future work will explore its use in supervised segmentation and in the context of conformal prediction.

\bibliographystyle{splncs04}
\bibliography{refs.bib}

%\section*{Annexes}
\appendix
\section{Proof Equation 10}
\label{preuve10}
%Cette représentation élimine naturellement le terme de normalisation global lorsque l'on considère les log-odds entre une classe et son complément. Sous l'hypothèse que les variables sont conditionnellement indépendantes sachant la méta-classe $\overline{C_j}$, on peut définir, pour chaque classe $C_j$, le log-odds a posteriori

This representation naturally eliminates the global term of normalization when considering the log-odds between a class and its complement. Under the assumption that the variables are conditionally independent given the metaclass $\overline{C_j}$, we can define, for each class $C_j$, the posterior log-odds

$$
r_j(D)
=
\log \frac{p_w(C_j\mid D)}
{p_w(\overline{C_j}\mid D)}.
$$

%D'après la structure additive du classifieur naïf de Bayes pondéré, ce score s'écrit
Based on the additive structure of the weighted naive Bayes classifier,
this score is expressed as
$$
r_j(D)
=
\log \frac{p(C_j)}{p(\overline{C_j})}
+
\sum_{m=1}^{M}\ell_m^{(j)}(D).
$$

%Ainsi, pour deux observations $D$ et $D'$, le terme a priori s'annule :
Thus, for two observations $D$ and $D'$, the a priori term cancels out:
$$
r_j(D)-r_j(D')
=
\sum_{m=1}^{M}
\left(
\ell_m^{(j)}(D)-\ell_m^{(j)}(D')
\right).
$$

Or,
$$
p_w(C_j\mid D)
=
\frac{1}{1+\exp(-r_j(D))}.
$$
%La fonction $h(t)=\log(1/(1+\exp(-t)))$ est $1$-Lipschitz. Par conséquent,

The function $h(t)=\log(1/(1+\exp(-t)))$ is $1$-Lipschitz. Therefore,
$$
\begin{aligned}
\left|
\log p_w(C_j\mid D)
-
\log p_w(C_j\mid D')
\right|
&\leq
\left|
r_j(D)-r_j(D')
\right| \\
&\leq
\sum_{m=1}^{M}
\left|
\ell_m^{(j)}(D)-\ell_m^{(j)}(D')
\right|.
\end{aligned}
$$

%En sommant cette inégalité sur toutes les classes, on obtient :
Summing this inequality over all classes, we get:
$$
\begin{aligned}
\Delta_1(D,D')
&=
\sum_{j=1}^{J}
\left|
\log p_w(C_j\mid D)
-
\log p_w(C_j\mid D')
\right| \\
&\leq
\sum_{j=1}^{J}\sum_{m=1}^{M}
\left|
\ell_m^{(j)}(D)-\ell_m^{(j)}(D')
\right| \\
&=
d_1^{\mathrm{odds}}(D,D').
\end{aligned}
$$

\end{document}